\documentclass{article}

\usepackage[margin=1in]{geometry}

\usepackage[utf8]{inputenc} 
\usepackage[T1]{fontenc}    
\usepackage{hyperref}       
\usepackage{url}            
\usepackage{booktabs}       
\usepackage{amsfonts}       
\usepackage{nicefrac}       
\usepackage{microtype}      
\usepackage{xcolor}         

\usepackage{natbib}

\usepackage{lipsum}
\usepackage{xcolor}         
\usepackage[linesnumbered,ruled]{algorithm2e}
\usepackage[title]{appendix}
\usepackage{amssymb}
\usepackage{csquotes}
\usepackage{bbm}
\usepackage{amsthm}
\usepackage{graphicx} 
\usepackage{algorithmic}
\usepackage{tablefootnote}
\usepackage{bm}

\usepackage{multirow}
\usepackage{adjustbox}
\usepackage{threeparttable}

\usepackage{wrapfig}

\usepackage{thm-restate}
\usepackage{dsfont}
\usepackage{comment}
\usepackage{array}
\newcolumntype{P}[1]{>{\centering\arraybackslash}p{#1}}

\newcommand{\norm}[1]{\left\lVert#1\right\rVert}
\newcommand{\idot}[1]{\langle#1\rangle}

\newtheorem{theorem}{Theorem}[section]

\newtheorem{corollary}{Corollary}[theorem]
\newtheorem{lemma}[theorem]{Lemma}
\newtheorem{remark}{Remark}
\newtheorem{definition}{Definition}

\usepackage{amsmath,amsfonts,bm}

\def\eqref#1{equation~\ref{#1}}
\def\Eqref#1{Equation~\ref{#1}}

\def\1{\bm{1}}

\def\va{{\bm{a}}}

\def\vh{{\bm{h}}}

\def\vr{{\bm{r}}}

\def\vu{{\bm{u}}}
\def\vv{{\bm{v}}}
\def\vw{{\bm{w}}}
\def\vx{{\bm{x}}}

\def\mA{{\bm{A}}}

\def\mD{{\bm{D}}}

\def\mI{{\bm{I}}}

\def\mM{{\bm{M}}}

\def\mW{{\bm{W}}}
\def\mX{{\bm{X}}}
\def\mY{{\bm{Y}}}
\def\mZ{{\bm{Z}}}

\DeclareMathAlphabet{\mathsfit}{\encodingdefault}{\sfdefault}{m}{sl}
\SetMathAlphabet{\mathsfit}{bold}{\encodingdefault}{\sfdefault}{bx}{n}

\def\gL{{\mathcal{L}}}

\def\gN{{\mathcal{N}}}

\def\sB{{\mathbb{B}}}

\def\sP{{\mathbb{P}}}

\def\sR{{\mathbb{R}}}

\newcommand{\E}{\mathbb{E}}

\newcommand{\Var}{\mathrm{Var}}

\DeclareMathOperator{\Tr}{Tr}

\title{Adversarial Noises Are Linearly Separable  \\for (Nearly) Random Neural Networks}

\makeatletter
\def\@fnsymbol#1{\ensuremath{\ifcase#1\or \dagger\or \ddagger\or
   \mathsection\or \text{*} \or \mathparagraph\or \|\or **\or \dagger\dagger
   \or \ddagger\ddagger \else\@ctrerr\fi}}
\makeatother

\author{
Huishuai Zhang\thanks{Microsoft Research Asia. {\tt huzhang@microsoft.com}}
\and
 Da Yu\thanks{Sun Yat-sen University. {\tt yuda3@mail2.sysu.edu.cn}}
\and 
Yiping Lu\thanks{Stanford University. {\tt yplu@stanford.edu}}
\and
Di He\thanks{Peking University. {\tt dihe@pku.edu.cn}} \footnote{The work was done when Da Yu and Yiping Lu were interns at Microsoft Research Asia.}
}

\begin{document}

\maketitle

\begin{abstract}
Adversarial examples, which are usually generated for specific inputs with a specific model, are ubiquitous for neural networks. In this paper we unveil a surprising property of adversarial noises when they are put together, i.e., adversarial noises crafted by one-step gradient methods are linearly separable if equipped with the corresponding labels. We theoretically prove this property for a two-layer network with randomly initialized entries and the \emph{neural tangent kernel} setup where the parameters are not far from initialization. The proof idea is to show the label information can be efficiently backpropagated to the input while keeping the linear separability. Our theory and experimental evidence further show that the linear classifier trained with the adversarial noises of the training data can well classify the adversarial noises of the test data, indicating that adversarial noises actually inject a distributional perturbation to the original data distribution. Furthermore,  we empirically demonstrate that the adversarial noises  may become \emph{less} linearly separable when the above conditions are compromised while they are still much easier to classify than original features. 
 
\end{abstract}

\section{Introduction}
Modern deep learning models have achieved great accuracy on vast intelligence tasks. However at the same time, they have been demonstrated vulnerable to adversarial examples, i.e., imperceptible perturbations can significantly change the output of a neural network at test time. This hinders the applicability of deep learning model on safety-critical tasks \citep{biggio2013evasion,Szegedy2013}.

Adversarial example is usually generated via finding a perturbed sample that maximizes the loss, i.e., 
\begin{flalign}
\arg \max_{\vx'\in \sB(\vx,\epsilon)} \ell(\vx', y;\theta). \label{eq:generation}
\end{flalign}
One popular method to solve \Eqref{eq:generation} is the \emph{fast gradient sign method (FGSM)} algorithm \citep{goodfellow2014explaining}, i.e.\footnote{Here we consider the corresponding constraint in \Eqref{eq:generation} is $l_2$ ball.}, 
\begin{flalign}
\vx^{adv} = \vx + \eta \nabla_x \ell(\vx,y;\theta) \label{eq:fgsm}
\end{flalign}
with a suitable step size $\eta$. Recent work   explains why the FGSM-style algorithm is able to attack deep models \citep{montanari2022adversarial, bartlett2021adversarial, bubeck2021single, daniely2020most}. They show that for random neural networks, the output is roughly linear around the input sample. Then the high-dimensional statistics tell that a random weight vector has small inner product with a given input while at the same time a small perturbation can sufficiently change the output.

From the formulation in \Eqref{eq:generation} and \Eqref{eq:fgsm}, it is obvious that the adversarial example is sample $(\vx,y)$ specific and model ($\theta$) specific and most existing theoretical and empirical researches of adversarial examples are mainly about this setting. In this paper, we study the adversarial noises from a  population's perspective and ask

\begin{center}
{ \emph{``What property do the adversarial noises exhibit when they are put together?''}}
\end{center}

Due to the complicated procedure of generating adversarial noises, one might think they must be  scattered quite casually and disorderly. However, surprisingly, we observe that  adversarial noises are almost linearly separable if they are equipped with the labels of the corresponding targeted samples, i.e., a new constructed dataset $\{(\text{adversarial noise } i, \text{ label } i)\}_{i\in[n]}$ is linearly separable (as shown in Figure ~\ref{fig:separability-train-test}). This finding is   important for us to better understand the behavior of adversarial noises.

In this paper, we first study why such phenomenon happens. Specifically, we consider the adversarial noises generated by the FGSM-style algorithm. We theoretically prove that for a randomly-initialized two-layer neural network, the adversarial noises are linearly separable. We further prove that the linear separability also holds for the neural tangent kernel (NTK) regime where the weights are near the initialization. The proof idea is to show the label information or the error of last layer, which are separable initially, can be efficiently backpropagated  to the input and then conceive a linear classifier that can classify these adversarial noises  perfectly. We deal with the correlation between forward and backward process via the Gaussian conditioning technique \citep{bayati2011dynamics, yang2020tensor, montanari2022adversarial}. Throughout the proofs, we spend many efforts to obtain  high probability bounds. Such high probability bounds are not only stronger than expectation bounds in theoretical sense  but also critical to make the linear separability claim valid for all adversarial noises of the whole dataset. This is in contrast with previous studies \citep{bubeck2019adversarial,montanari2022adversarial} that are to understand example-specific property of adversarial noises, e.g. why an adversarial noise is imperceptible but able to attack successfully.

The theory indicates that the linear separability of adversarial noises actually are generalizable to the test set, which is also verified in Figure~\ref{fig:separability-train-test}. That is, a linear classifier trained on the adversarial noises of the training data points can  well classify the adversarial noises of the test data points as long as they follow the same procedure of generation. This means the adversarial noises actually inject a distributional perturbation to the original data distribution.
\begin{figure}
\centering
  \includegraphics[width=0.7\linewidth]{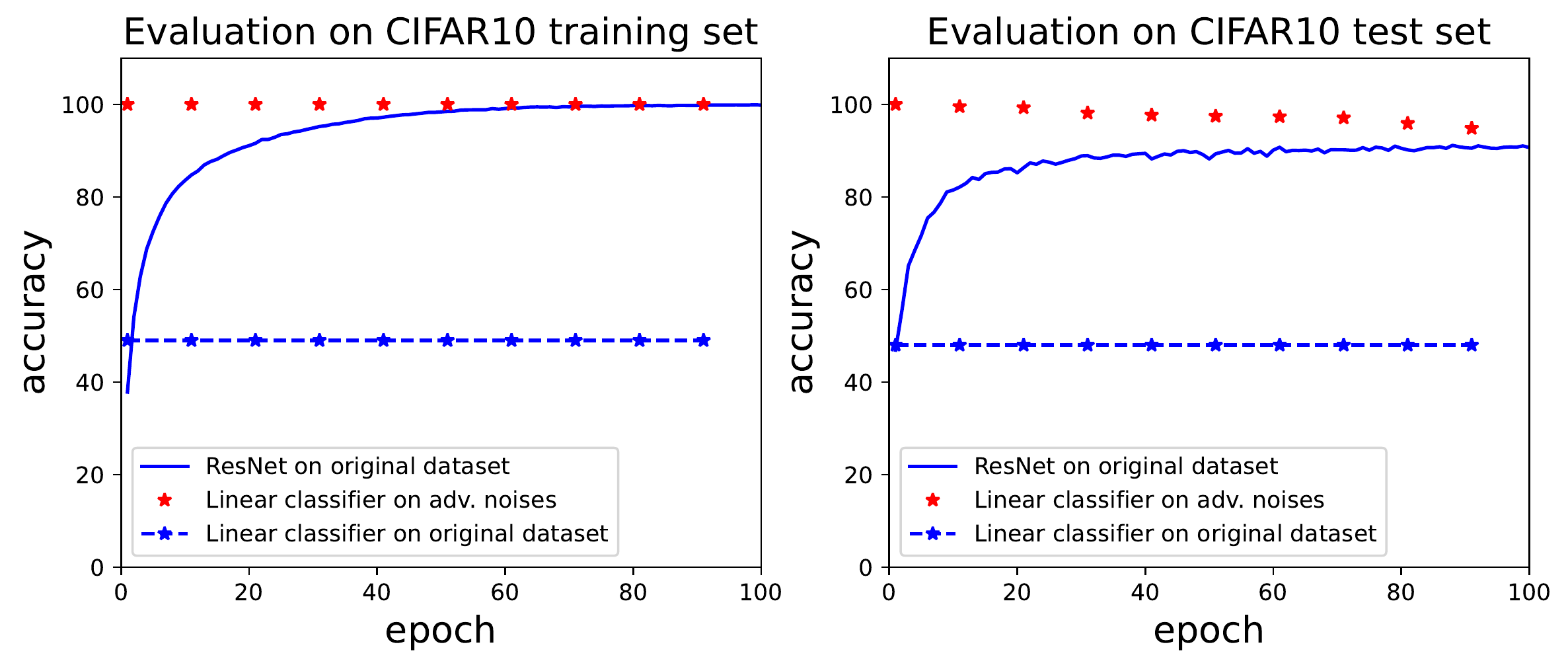}
  \caption{Training and test accuracy of \textbf{linear models} on adversarial noises, which are generated with ResNet-18 on CIFAR-10 over a standard training process of SGD with  $lr=0.001$.}
  \label{fig:separability-train-test}
\end{figure}

We also empirically explore the property of adversarial noises beyond the theoretical regime, especially for the case where the neural network is trained with a large learning rate,  the case of other adversarial noise generation algorithms, and  the case of adversarially trained models. Although the adversarial noises are not perfectly linearly separable in these wild scenarios, a consistent message is that they are  much easier to fit than original features, i.e., a linear classifier on  adversarial noises can achieve much higher accuracy  than  the best linear classifier on the original dataset. 


Overall, our contribution can be summarized as follows.
\begin{itemize}
    \item We unveil a surprising phenomenon that adversarial noises are almost linearly separable. 
    
    \item We theoretically prove the linearly separability of adversarial noises for (nearly) random two-layer networks.

    \item We show that the linear separability of adversarial noises may be compromised when going beyond the theoretical regime, but they are still much easier to classify than original features.
\end{itemize}

\subsection{Related work}
There are some explanations why adversarial examples exist, e.g., the deep network classifiers being too linear locally because of ReLU like activations \citep{goodfellow2014explaining}, the boundary tilting hypothesis that the classification boundary is close to the submanifold of the training data \citep{tanay2016boundary}, the isoperimery argument \citep{fawzi2018anyclassifier, shafahi2019adversarial} and the dimpled manifold model \citep{shamir2021dimpled}. There are also theoretical researches on the difficulty of adversarial learning difficulty, e.g., robust classifier requiring much more training data \citep{schmidt2018adversarially} and the computational intractability of building robust classifiers \citep{bubeck2019adversarial},  adversarial examples being features not bugs \citep{ilyas2019adversarial}. 

One related concept is the \emph{label leakage} \citep{kurakin2016adversarial} that  adversarial examples are crafted by using true label information  in the single-step gradient methods and  hence may be easier to classify. Our results greatly extend/verify this concept by showing the adversarial noises are linearly separable.

Our finding is also related with the concept \emph{shortcut learning} \citep{beery2018recognition,niven2019probing,geirhos2020shortcut} that deep models may rely on shortcuts to make predictions. Shortcuts are spurious features that are correlated with the label but not in a causal way. In this work, we show the linear separability makes adversarial noises perfect shortcut, which may hinder the classifier learns true features in the adversarial training.  Our study is inspired by the finding that the data poisoning for availability attack is adding simple features \citep{yu2021indiscriminate} to the training data. We focus on the adversarial noises and analyze their linear-separability theoretically.

\section{Problem Setup and Notations}

We study the distribution of adversarial noises of neural networks. Although the study is not constrained to specific networks, we analyze a simplified model to ease the technical exposure.

Specifically, we consider a two-layer neural network with input dimension $d$ and width $m$:
\begin{equation}
    f(\vx;\va, \mW) = \va^\top \sigma(\mW\vx), \label{eq:2layerNN}
\end{equation}
where $\sigma$ is the ReLU activation function which is applied coordinate-wisely, input  $\vx \in \sR^d$, $\mW \in \sR^{m\times d}$, and readout layer $\va\in \sR^m$. The network parameters are initialized as follows.  Each entry of $\mW$ is independently generated from $\gN(0, 1/d)$, and each entry of $\va$ is  independently generated from $\gN(0,1/m)$. Moreover, $\mW$ and $\va$ are independent from each other. We use a new notation $\theta$ to represent the whole trainable parameters in the network, i.e., here $\theta = \{\mW, \va\}$.

We consider binary classification task with a dataset $\{(\vx_i, y_i)\}_{i\in [n]}$, where $\vx_i\in \sR^d$ and $y_i\in \{-1,+1\}$ for $i=1, ..., n$. We use a negative log sigmoid loss, i.e.,
\begin{flalign}
 \ell(y f(\vx) )= -\log s(yf(\vx)),
\end{flalign}
where $s(z) = \frac{1}{1+e^{-z}}$ is the sigmoid function. 

We consider the one-step gradient method to generate  the adversarial noise, i.e., 
\begin{flalign}
 \vr_x = \frac{\partial \ell(\vx)}{\partial \vx} = - (1-s(yf(\vx))) y\nabla_x f(\vx), \label{eq:advnoise}
\end{flalign}
where  $\nabla_x f(\vx)$ is the gradient with respect to the input and we may omit the subscript when it is clear from the context. For the two-layer neural network (\Eqref{eq:2layerNN}), this gradient is given by 
\begin{flalign}
\nabla f(\vx) = \mW^\top \mD_x \va,
\end{flalign}
where $\mD_x \in \sR^{m\times m}$ is a diagonal matrix and the diagonal entries are given by  $\sigma'(\mW\vx)$. The adversarial example is given by 
\begin{flalign}
\vx^{adv} = \vx + \eta \vr_x, \label{eq:advexample}
\end{flalign}
where $\eta$ is step size has magnitude $O(1)$. Here we assume the ball constraint in \Eqref{eq:generation} is measured in $l_2$ distance and hence the projection can be removed. It is  interesting and important to extend the analysis to other distances which are empirically verified in Section \ref{sec:exp}

 We next state the mathematical definition of linear separability for a binary-label dataset.

\begin{definition}[Linearly separable]
We say a set $\{\xi_i, y_i\}_{i\in[n]}$ with $y_i \in \{+1,-1\}$ linearly separable if $\exists \vv$ such that $\forall i : \langle \vv,y_i\xi_i\rangle >0$.
\end{definition}

\textbf{Notations.} In the seuqel, we use $\Vert \vx\Vert$ to denote the $\ell_2$ norm of a vector $\vx$, We use $\norm{\mM}_2$ and $\norm{\mM}_F$ to denote the spectral norm and the Frobenius norm of a matrix $\mM$, respectively. The learning process is to minimize the average loss  $\gL(\theta)=\sum_{i=1}^{n}\ell(\theta;\vx_i,y_i)/n$.

Besides, we also define the following notations to describe the bounds we derived. We write $f(\cdot)=\mathcal{O}(g(\cdot))$, $f(\cdot)=\Omega(g(\cdot))$ to denote $f(\cdot)/g(\cdot)$ is upper or lower bounded by a positive constant. We use $f(\cdot)=\Theta(g(\cdot))$ to denote that $f(\cdot)=\Omega (g(\cdot))$ and $f(\cdot)=\Omega (g(\cdot))$. 

\section{Provable Linear Separability of Adversarial Noises}

It has been shown that the adversarial noises generated by \Eqref{eq:advnoise} are small while being able to change the output significantly \citep{montanari2022adversarial,bartlett2021adversarial}. We next show that the adversarial noises  exhibit surprising linearly-separable phenomenon when put together. In this section we first analyze why such phenomenon exists for randomly initialized network. Then we extend the analysis to the NTK setting.

\subsection{Linear Separability at Initialization}
We next claim that for a two-layer network at its initialization, the  adversarial noises are linearly separable if equipped with corresponding labels, i.e., $\{\vr_{x_i}, y_i\}_{i=1}^n$ is linearly separable.

\begin{theorem}\label{thm:initialization-results}
For the two-layer network given by \Eqref{eq:2layerNN} and the adversarial noises $\{\vr_{x_i}\}_{i=1}^n$ generated by \Eqref{eq:advnoise}, there exists $\vv$ such that $\forall i\in[n], \langle\vv, y_i\vr_{x_i}\rangle > 0$ with high probability. Specifically $\vv = - \mW^\top \va$ serves this purpose with probability at least $1-3Cn(e^{-c_1 d}+e^{-c_2 m})$ where $C, c_1, c_2$ are some constants.
\end{theorem}

The adversarial noise  in \Eqref{eq:advnoise}  is  $\vr_{x_i} = -(1-s(yf(\vx)))y_i\nabla f(\vx_i)$, where $1-s(yf(\vx))>0$ because of the sigmoid function. Hence it is sufficient to show that $\langle -\vv, \nabla f(\vx_i)\rangle > 0$ holds for all $i$. Next we give a proof outline and the full derivation is deferred to Appendix \ref{app:thm:initialization-results}.

\begin{proof}[Proof Outline]
To give an intuitive idea why the claim is probably true, for a generic input $\vx$ and $\vv=-\mW^\top \va$, we calculate the expectation and the variance of $\langle-\vv, \nabla f(\vx)\rangle > 0$ for a simplified case: {$\mD_x$ is  random and independent from all others $\{\mW, \va, \vx\}$.} We are safe to ignore the subscript for this case.

Because of the property of ReLU activation, we further assume that the diagonal entries of $\mD$ are  independently and randomly sampled from $\{0,1\}$ with equal probability, i.e., for all $k\in[m]$
\[
\mD(k,k) = \begin{cases}
0, \;\; \text{ with probability } 0.5, \\
1, \;\; \text{ with probability } 0.5.
\end{cases}
\]
Then we can compute $\E \langle -\vv, \nabla f(\vx)\rangle$ and $\Var \langle -\vv, \nabla f(\vx)\rangle$ as follows,
\begin{flalign}
&\E \langle -\vv, \nabla f(\vx)\rangle = \E [\va^\top\mW \mW^\top \mD \va]=\E \left[\Tr(\va\va^\top\mW \mW^\top \mD)\right] \nonumber\\
&= \Tr\left(\E[\va\va^\top]\E[\mW \mW^\top] \E[\mD]\right) = \Tr \left( \frac{1}{m} \mI_{m\times m} \cdot \mI_{m\times m} \cdot \frac{1}{2} \mI_{m\times m}\right)= \frac{1}{2}, 
\end{flalign}
and 
\begin{flalign}
&\Var \langle -\vv, \nabla f(\vx)\rangle = \E \left(\va^\top\mW \mW^\top \mD \va\right)^2 - \left(\E [\va^\top\mW \mW^\top \mD \va]\right)^2, \nonumber\\
&=\left(\frac{1}{4} + \frac{5}{4m} + \frac{1}{d} + \frac{2}{md}\right) -\frac{1}{4} = \frac{5}{4m} + \frac{1}{d} + \frac{2}{md}
\end{flalign}
We note that the computation of $\E \left(\va^\top\mW \mW^\top \mD \va\right)^2$ is quite complicated and heavily relies on the property of $\va, \mW$ being Gaussian and the independence between $\va,\mW$ and $\mD$. By Chebyshev inequality, we can show that $\langle -\vv, \nabla f(\vx)\rangle > \frac{1}{2}- \delta$ with  probability at least $1-\frac{2}{\delta^2 d}$ assuming that $m>1.2d+2$. Taking the union bound, we can prove the claim holds with probability $1-n \frac{2}{\delta^2 d}$. Thus, it requires $d\gg n$ to claim that the theorem holds with high probability. To obtain a tighter bound, it requires more elaborate concentration inequality, which is deferred to Appendix \ref{app:thm:initialization-results}.

Next we consider the case {where $\mD_x$ are exactly $\sigma'(\mW \vx)$.} 

This makes the analysis a bit harder as the $\mW, \mW^\top$ and $\mD_x$ are correlated. We prove the claim via the technique of probability concentration and Gaussian conditioning \citep{yang2020tensor,montanari2022adversarial}. We use a lemma as follows.

\begin{lemma}[Lemma 3.1 in \citep{montanari2022adversarial}] \label{lem:gaussian-conditioning}
Let $\mX \in \sR^{m \times d}$ which has i.i.d. standard Gaussian entries, and $\mA_1 \in \sR^{k_1 \times m}, \mA_2 \in \sR^{d \times k_2}$. Let $\mY = h_1(\mA_1 \mX, \mX \mA_2, \mZ_1)$ with $\mZ_1$ independent of $\mX$, $\mA_2 = h_2(\mA_1 \mX, \mZ_2)$ with $\mZ_2$ independent of $\mX$. We assume that $(\mA_1, \mZ_1, \mZ_2)$ is independent of $\mX$. Then there exists $\tilde{\mX} \in \sR^{m \times d}$ which has the same distribution with $\mX$ and is independent of $\mY$, such that
\begin{align*}
	\mX = \Pi_{\mA_1}^{\perp}\tilde{\mX} \Pi_{\mA_2}^{\perp} + \Pi_{\mA_1}^{\perp}{\mX} \Pi_{\mA_2} + \Pi_{\mA_1}{\mX} \Pi_{\mA_2}^{\perp} + \Pi_{\mA_1}{\mX} \Pi_{\mA_2},
\end{align*} 
where $\Pi_{\mA_1} \in \sR^{m \times m}$ is the projection operator projecting onto the subspace spanned by the rows of $\mA_1$, $\Pi_{\mA_2} \in \sR^{d \times d}$ is the projection operator projecting onto the subspace spanned by the columns of $\mA_2$, and $\Pi_{\mA_1}^{\perp} := \mI_m - \Pi_{\mA_1}$, $\Pi_{\mA_2}^{\perp} := \mI_d - \Pi_{\mA_2}$. 
\end{lemma}	
The proof of Lemma~\ref{lem:gaussian-conditioning} is in Appendix A.1 of \citet{montanari2022adversarial}. By using Lemma~\ref{lem:gaussian-conditioning}, where  plugging in $\mX\leftarrow \mW$, $\mA_1 \leftarrow 0, \mA_2 \leftarrow \vx$, $\mY \leftarrow \mD_x$ and $\Pi_x = \frac{1}{d}\vx\vx^\top$, we have
\begin{flalign}
\va^\top \mW \mW^\top \mD_x \va
=& \frac{1}{d}\va^\top \mW \vx\vx^\top \mW^\top \mD_x\va + \va^\top \tilde{\mW}\Pi_x^\perp  \tilde{\mW}^\top \mD_x \va \nonumber\\
=& \frac{1}{d}\va^\top \mW \vx\vx^\top \mW^\top \mD_x \va + \va^\top \bar{\mW}\bar{\mW}^\top \mD_x \va, \label{eq:term-with-x}
\end{flalign}
where $\tilde{\mW}$ has the same marginal distribution as $\mW$ and is independent of $\mD_x, \va$, and $\bar{\mW}\in \sR^{m\times (d-1)}$ has i.i.d. Gaussian entries with mean 0 and variance $\frac{1}{d}$ and is independent of $\mD_x, \va$. 

For the first term in \Eqref{eq:term-with-x}, let $\vh = \mW\vx$, then with high probability $\norm\vh \approx \sqrt{m}$ and $\norm{\mD_x\vh} \approx \sqrt{m/2}$. Given $\vh$, we have $\va^\top\vh \sim \gN(0, \norm\vh^2/m)$ and  $\vh^\top\mD_x\va \sim \gN(0, \norm{\mD_x\vh}^2/m)$. Then we have 
\begin{flalign}
\frac{1}{d}\left|\va^\top \mW \vx\vx^\top \mW^\top \mD_x \va\right| =\frac{1}{d}\left|\va^\top \vh \vh^\top \mD_x \va\right| 
= \frac{1}{d}\left|\va^\top \vh|\cdot|\vh^\top \mD_x \va\right|. \label{eq:term-with-x1}
\end{flalign}

We can bound the right hand side of \Eqref{eq:term-with-x1} with high probability by using the following two lemmas. 


\begin{lemma} \label{lem:hnormbound}
Suppose $\norm{\vx} = \sqrt{d}$ and $\mW$ is a Gaussian matrix with entry variance $1/d$. Let $\vh = \mW\vx$, then we have 
\begin{flalign}
\sP\{\norm{\vh}^2 < 2 m\} >1-e^{m/7}.
\end{flalign}
\end{lemma}
The proof of this lemma is based on the tail bound of $\chi^2$ distribution.

\begin{lemma} \label{lem:innerproduct}
Suppose $\norm{\vx} = \sqrt{d}$, $\mW$ is a Gaussian matrix with entry variance $1/d$ and $\va$ is a Gaussian vector with entry variance $1/m$. Let $\vh = \mW\vx$ and $\mD_x = \sigma'(\mW\vx)$. Then with probability at least $1-e^{-m/7}-4e^{-c_2d/4}$, we have 
\begin{flalign}
|\va^\top\vh| < \sqrt{c_2d}, \;\;\; |\va^\top\mD_x\vh| < \sqrt{c_2d},
\end{flalign}
where $c_2$ is some constant.
\end{lemma}

Thus if choosing $c_2 <1/64$, we prove that the \Eqref{eq:term-with-x1} smaller than $1/64$ with probability at least $1- n(e^{-m/7}+e^{-d/256})$.

For the second term in \Eqref{eq:term-with-x}, we can use the result for the case where $\mD_x$ is indepdent of $\mW\vx, \va$ and have a lower bound for it. Combining these two parts together, we prove the Theorem \ref{thm:initialization-results} with high probability.
\end{proof}

We have shown that the linear separability of adversarial noises for network at its random initialization. Then one question is whether the adversarial examples are linearly separable, i.e., if there exists one $\vv'$ such that $\langle \vv', y_i \vx_i^{adv}\rangle > 0 $ for all $i\in [n]$. This is true if the input dimension and the network width are much larger than the number of input samples. In this case we can find a linear classifier that lives in a subspace perpendicular to the linear space spanned by $\{\vx_i\}_{i=1}^n$. 
\begin{corollary}\label{cor:advresult}
For the two-layer network defined in \Eqref{eq:2layerNN} and the adversarial samples given by $\vx_i^{adv} = \vx_i + \eta \vr_i$, if $d>poly(n)$ there exists $\vv'= -\Pi_\mX^\perp \mW^\top \va$ such that $\forall i: \langle \vv', \vx_i^{adv} \rangle > 0$ with high probability. 
\end{corollary}

\begin{proof}
The idea is that we can make the classifier staying in the orthogonal subspace of $\mX\mX^\top$ while can still linearly separates the adversarial samples. 

We note that $\idot{\vv', \vx^{adv}} = \idot{\vv', \vx} + \idot{\vv', -\eta (1-p)\nabla f(\vx)} =  \idot{\vv', -\eta (1-p)\nabla f(\vx)}$. We next prove with high probability
\begin{flalign}
 \va^\top \mW\Pi_\mX^\perp\mW^\top \mD_x \va >0. 
\end{flalign}
The above is indeed true because we can use Gaussian conditioning, i.e., 
\begin{flalign}
 \va^\top \mW\Pi_\mX^\perp\mW^\top \mD_x \va  = \va^\top \bar{\mW}\bar{\mW}^\top \mD_x \va,
\end{flalign}
where $\bar{\mW}\in \sR^{m\times (d-n)}$ with $i.i.d.$ Gaussian entries with mean $0$ and variance $1/d$. Then following the argument in the proof of Theorem \ref{thm:initialization-results}, we complete the proof.
\end{proof} 

\begin{remark}
If $d$ is not larger than $n$, then there may not exist a valid $v'$ in Corollary \ref{cor:advresult}.
\end{remark}
In this setting, there is not enough randomness in $\mW$ to exploit. One possible choice is to increase the energy of the adversarial signal (by increasing the step size of \Eqref{eq:advexample}) to overcome the effect of the original input $\vx$. By choosing $\eta = d^{1/4}$, the adversarial noise is still small compared with the original signal, i.e., $\frac{\norm{\vx^{adv}-\vx}}{\norm{\vx}} = O(d^{-1/4})$ but the effect of the adversarial noise overweighs that of the original signal, i.e., 
$\frac{|\idot{-\mW^\top\va,  \vx^{adv}-\vx}|}{|\idot{-\mW^\top\va, \vx}|} = O(d^{1/4})$. Thus the adversarial examples may still be linearly separable in this case.

 
\subsection{Linear Separability in NTK Regime}

We have established the linear separability of adversarial noises for two-layer networks at initialization. 
In this section, we study the behavior of the adversarial noises when the network is slightly trained, i.e., the weights are not far from initialization. By the convergence theory of training neural network in \emph{Neural Tangent Kernel} (NTK) regime,  the network parameter  can  fit the training data perfectly even in a small neighborhood around initialization as long as the width of the network is large enough  \citep{jacot2018neural,allen2018convergence, du2018gradient, chizat2018global,  zou2018stochastic, zhang2019convergence}. A typical result reads as follows, which we adapt to our notations.



\begin{lemma}[Theorem 1 in \citep{allen2018convergence}] \label{lem:ntkconvergence}
Suppose a two-layer neural network defined by \Eqref{eq:2layerNN} and a distinguishable dataset with $n$ data points. If the network width $m \ge \Omega(poly(n)\cdot d)$, starting from random initialization $\theta$, with probability at least $1-e^{\Omega(\log^2 m)}$, then gradient descent with learning rate $\Theta\left(\frac{d}{poly(n)}\right)$ finds $\{\mW^*, \va^*\}$ such that $\gL(\mW^*, \va^*) \le \epsilon$ and $\norm{\mW^*-\mW}_2\le \frac{1}{\sqrt{m}}$ and $\norm{\va^*-\va}\le \frac{1}{\sqrt{m}}$.
\end{lemma}


Based on this result of NTK convergence, we can see that when the loss is minimized, the learned parameters are still very close to the initialization especially as the width becomes large. Thus, it is possible for us to show that the adversarial noises at the NTK solution are linear separable. 
\begin{theorem}\label{thm:ntk-results}
For the two-layer network defined in \Eqref{eq:2layerNN}, the NTK solution $\{\mW^*, \va^*\}$ satisfying Lemma \ref{lem:ntkconvergence}, and the adversarial noises $\{\vr_i\}_{i=1}^n$ given by \Eqref{eq:advnoise}, there exists $\vv$ such that $\forall i: \langle \vv, y_i\vr_i\rangle > 0$. Specifically $\vv = -\mW^\top \va$ serves this purpose  with high probability at least $1-Cn e^{-\Omega({m/\log^2 m}) - cd }$ for some constants $C,c$.
\end{theorem}
\begin{proof}[Proof Outline]
The proof relies on that the NTK solution is very close to the initialization. We ignore the $1-s(yf(\vx))$ term and only  calculate
\begin{flalign}
&\langle \mW^\top \va, \nabla_x f(\vx;\mW^*,\va^*)\rangle = \va^\top \mW \mW^{*\top} \mD_x^* \va^* \nonumber\\
=& \va^\top \mW \mW^{\top} \mD_x \va + \va^\top \mW (\Delta\mW)^{\top} \mD_x \va +\va^\top \mW \mW^{\top} (\Delta\mD_x) \va + \va^\top \mW \mW^{\top} \mD_x (\Delta \va) \nonumber\\
& + \va^\top \mW\left(\Delta\mW^{\top} \Delta\mD_x \va + \Delta\mW^{\top} \mD_x \Delta\va +\mW^{\top} \Delta\mD_x \Delta\va + \Delta\mW^{\top}  \Delta\mD_x \Delta \va\right)\label{eq:ntk-expansion}
\end{flalign}
where $\Delta \mW = \mW^* - \mW, \Delta \va = \va^* - \va$ and $\Delta \mD_x = \mD_x^* - \mD_x$. 

For the first term of \Eqref{eq:ntk-expansion}, by Theorem~\ref{thm:initialization-results} we have with high probability
\begin{flalign}
 \va^\top \mW \mW^{\top} \mD_x \va>1/32. \label{eq:ntk-expansion1}
\end{flalign}
For the second term of \Eqref{eq:ntk-expansion}, we note that with high probability $\norm{\mD_x\va} \in  (\frac{1}{2}-\delta , \frac{1}{2}+\delta)$ and $\norm{\va^\top \mW} \in (1-\delta, 1+\delta)$, and hence 
\begin{flalign}
 |\va^\top \mW (\Delta\mW)^{\top} \mD_x \va|\le O(\frac{1}{\sqrt{m}}). \label{eq:ntk-expansion2}
\end{flalign}

{We next bound the third term of \Eqref{eq:ntk-expansion}}. From the convergence proof in the NTK regime, we have $\norm{\Delta \mD_x}_0 < \frac{m}{\log^2 m}$ with probability at least 1 -$e^{-\Omega(m/\log^2 m)}$ \citep[Lemma 8.2]{allen2018convergence}. Hence with high probability $\norm{(\Delta \mD_x) \va} \le \frac{1}{\log m}$. We need the following lemma to get an overall high probability bound.
\begin{lemma}[Lemma~7.3 in \citep{allen2018convergence}]
For all sparse vectors $\vu$ with $\norm{\vu}_0\le O(\frac{m}{\log^2 m})$, we have with probability at least $1-e^{-\Omega(m)}$, 
\begin{flalign}
 |\va^\top \mW \mW^{\top} \vu| \le 2\norm{\vu}.
\end{flalign}
\end{lemma}

Thus, we have 
\begin{flalign}
 |\va^\top \mW \mW^{\top} (\Delta \mD_x \va)|\le \frac{1}{\log m}. \label{eq:ntk-expansion3}
\end{flalign}

{We next bound the fourth term of \Eqref{eq:ntk-expansion}}
\begin{flalign}
 |\va^\top \mW \mW^{\top} \mD_x (\Delta \va)|&\le \norm{\va^\top \mW} \cdot \norm{\mW^\top}_2 \cdot \norm{\mD_x (\Delta \va)} \le 2 \cdot \sqrt{\frac{m}{d}} \cdot \frac{1}{\sqrt{m}} = O(\frac{1}{\sqrt{d}}). \label{eq:ntk-expansion4}
\end{flalign}

For the higher order terms in \Eqref{eq:ntk-expansion}, we can similarly bound them one by one. 

Hence by combining bounds in (\ref{eq:ntk-expansion1}), (\ref{eq:ntk-expansion2}), (\ref{eq:ntk-expansion3}) and (\ref{eq:ntk-expansion4}) together and taking the union bound, we complete the proof.
\end{proof}

Beyond the NTK setting, we discuss the possible extensions here. First it is possible to extend the current results to the multi-layer neural network setting, as it has been demonstrated that a multi-layer neural network near its initialization also behaves like linear function with respect to the input \citep{allen2018convergence,bubeck2021single, montanari2022adversarial}. Second, it is desirable if one can extend the result to the FGSM  with respect to the $l_\infty$ constraint. The extension towards this direction is not obvious based on current technique. One main difficulty is that the sign operation in FGSM would break the Gaussian property of the adversarial noises. One can hardly exploit the Gaussian conditioning to give a proof.  We next do empirical experiments and verify the distributional property of the adversarial noises for all these settings.   
\section{Linear Separability of Adversarial Noises in Practice} \label{sec:exp}

In this section, we empirically verify the linear separability of adversarial noises. Specifically, we will first verify our theories' prediction that the adversarial noises are indeed linearly separable for neural network at/near its random initialization. Then we go beyond the theoretical regime and explore the case where the networks are sufficiently trained and the case where the adversarial noises are generated with multiple-step PGD. We next describe the general setup of our experiments. 

\subsection{General Setup of Experiments}

The target model architecture is the ResNet-18 model in \citet{he2016identity}. We train the target models on the CIFAR-10 dataset \citep{cifar} with standard random cropping and flipping as data augmentation.  All models are trained for 100 epochs with a batchsize of 128. 

To test the linear separability of adversarial noises, we train linear models that use the generated adversarial noises as input and  the labels of corresponding target examples as the their labels. All perturbations are flattened into one-dimensional vectors. The higher the training accuracy of linear models, the better the linear separability. All linear models are trained for 50 steps with the L-BFGS optimizer \citep{liu1989limited}. All the experiments are run with one single  Tesla V100 GPU.


\subsection{Adversarial Noises Within Theoretical Regime}
\label{subsec:exp_theoretical}

We first verify our theoretical findings. For the initialization setup, we use random Gaussian (Kaiming initialization \citep{he2016deep}) to initialize the neural network and test the linear separability of its adversarial noises.

We do not directly work with the neural tangent kernel. Instead we use a small constant learning rate 0.001 to mimic the case that the model is close to initialization across training. We take a snapshot of the model every 10 epochs of training. We generate the adversarial noises with respect to each snapshot and then train a linear classifier on them accordingly.

All adversarial noises are generated with  one-step FGSM.  In addition to the training accuracy of linear models, we also report how well these linear models ``generalize'' to the adversarial noises on test data points, the so-called \emph{test accuracy} of the linear models on adversarial noises.

We plot the result in Figure~\ref{fig:separability-train-test}. As the model is trained, the ResNet's  accuracy on the original training data increases to 100\% steadily. A linear model cannot fit the original training data well but a linear model can fit the adversarial noises perfectly from the initialization to the end of the training. This finding confirms that our theoretical findings do hold in practice.

We also observe that the ability of the linear classifier generalizes to the ``test data'', i.e., the linear model trained on adversarial noises of the training data performs well on the adversarial noises of the test data. This finding implies that although the adversarial noises  are designed with respect to specific samples, they actually introduce a new distribution on $(\vx;y)$ to perturb the original data distribution. This new perspective may inspire new ways to defend against the adversarial noises.


%





\subsection{Adversarial Noises Beyond Theoretical Regime}
\label{subsec:exp_empirical}

In this section, we go beyond the theoretical regime and see how the adversarial noises behave for models/algorithms in the wild.

\textbf{Large learning rate.} We first test the linear separability of adversarial noises  for network that is sufficiently trained with a large learning rate. Specifically for the same ResNet-18 and CIFAR-10, we use learning rate $lr=0.1$ instead of 0.001 in previous subsection so that the model is no longer close to initialization as the training proceeds.  Similarly, we take a snapshot of the model every 10 epochs of training. We generate the adversarial noises with respect to each snapshot and then train a linear classifier on them accordingly. We plot the result in Figure~\ref{fig:separability-train-test-lr0.1}. We can see that  indeed, for this setting, the linear separability of adversarial noises becomes weaker. 

Apart from the reason that the weights move far away from initialization, we  identify that the errors at the last layer become less separable as the training loss becomes small.  At initialization, all the output activation is random and the only signal in the last layer gradient is the label. After training, we gradually learn label's information which makes the signal in the last layer gradient is not that informative and separable. We can alleviate this effect by tuning the softmax temperature of generating adversarial noises as shown in Figure~\ref{fig:separability-train-test-lr0.1}. We note that the temperatures of softmax are only used for generating  adversarial noises while not affecting the training of ResNet models. The larger $T$, the more uniform the softmax output.

For the case of large learning rates, even though the adversarial noises are not perfectly linearly separable, they are easier to classify than the original features, e.g., the accuracy of linear classifier on adversarial noises are higher than that on original features (see Figure~\ref{fig:separability-train-test-lr0.1}). Moreover, if we replace the linear classifier with a  two-layer neural network, the adversarial noises can still be perfectly fit.

 \begin{figure}
\centering
    \begin{minipage}{0.49\textwidth}
        \begin{center}
        \centerline{\includegraphics[width=1.\columnwidth]{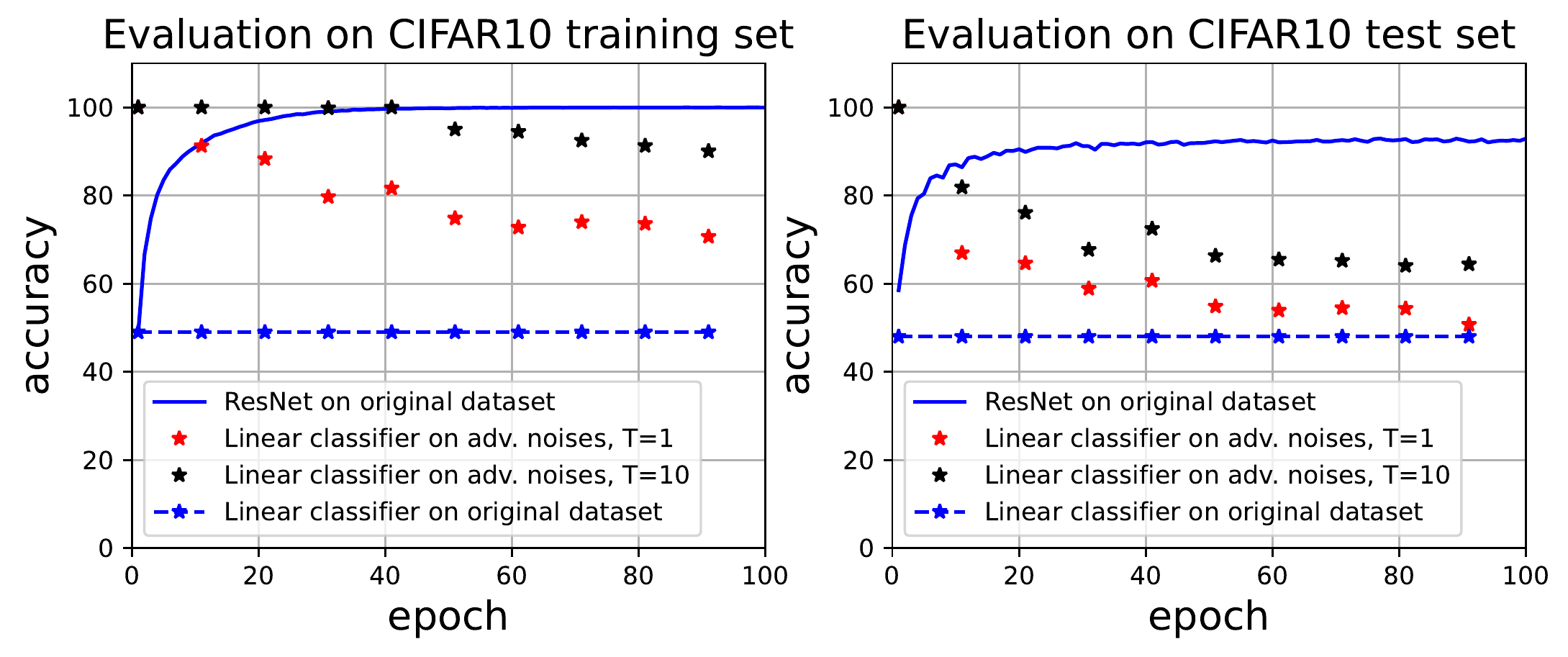}}
        \caption{Training and test accuracy of \textbf{linear models} on adversarial noises. The noises are generated by using ResNet-18 models trained with \textbf{lr=0.1} on  CIFAR-10 and using two  softmax temperatures $T=1$ and $T=10$.}
\label{fig:separability-train-test-lr0.1}
        \end{center}
    \end{minipage}
    \hfill
    \begin{minipage}{0.49\textwidth}
        \begin{center}
        \centerline{
\includegraphics[width=1.0\columnwidth]{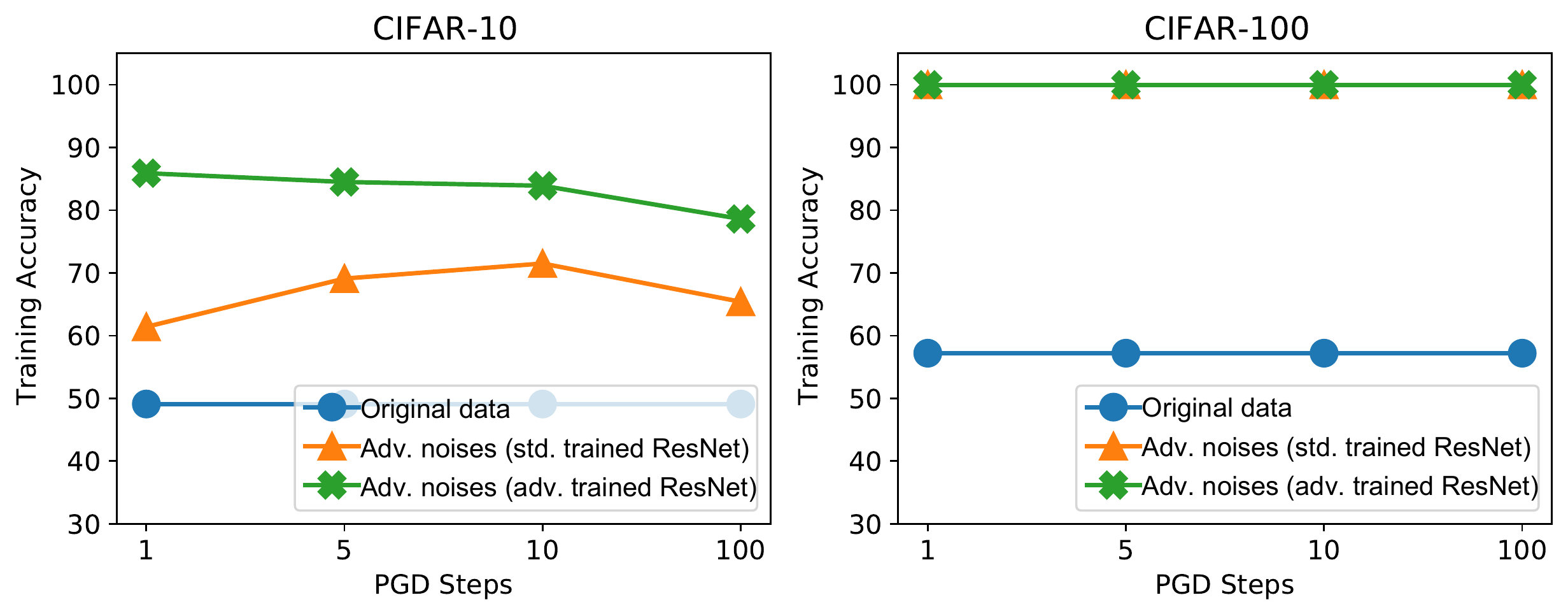}}
        \caption{Training accuracy of linear models on adversarial noises generated with standardly/adversarially trained models.  The blue line is training accuracy  on original data. }
  \label{fig:pgd-models}
        \end{center}
    \end{minipage}
\vspace{-4mm}
\end{figure}




\textbf{Multi-step PGD and adversarially trained models.} In this part, we consider the adversarial noises generated by the final models trained either standardly or adversarially. For adversarial training, we adopt the setup in \citet{madry2018towards} that uses 7 steps of PGD with a stepsize of $2/255$. 

For generating the adversarial noises, we test PGD with 5, 10, and 100 steps.  In addition to CIFAR-10, we also experiment with the CIFAR-100 dataset. We also plot the linear separability of clean data for a comparison.  The results are in Figure~\ref{fig:pgd-models}.

We can see that the number of PGD steps does not affect the linear separability much. For both datasets, the adversarial noises are easier to fit than the original data for both standardly-trained and adversarially-trained models. 

For the CIFAR-10 dataset, the adversarially trained model has substantially better linear separability than the standardly trained model. We speculate that the adversarially trained model has larger training losses and hence larger error at the last layer, which shares similar effect to tuning the temperature as shown in Figure~\ref{fig:separability-train-test-lr0.1}. For the CIFAR-100 dataset, we observe the linear models achieve 100\% accuracy for both standardly and adversarially trained models and linear models also achieve higher accuracy on original CIFAR-100 than the original CIFAR-10 desipite CIFAR-100 is  more challenging. This may be because the linear model for CIFAR-100 has 10 times more parameters than that of CIFAR-10 (because the number of parameters of the linear model depends on the number of classes).

Although the adversarial noises may not be perfectly linearly separable for these wild scenarios, one consistent message is that the adversarial noises are still much easier to fit than original data. The linear classifier still generalizes to the adversarial noises on test data to some extend, which indicates adversarial noises inject distributional perturbation to the original data distribution. There are many other settings to explore, e.g., different model structures, different training algorithms (standard training or adversarial training), which are suitable for future study. 










\section{Discussion and Conclusion}
In this paper, we unveil a phenomenon that adversarial noises are almost linearly separable if equipped with target labels. We theoretically prove why this happens for two-layer randomly initialized neural networks. One key message is that the adversarial noises are easy to fit for no matter nearly random network or fully trained network. Such easy-to-fit property of adversarial noises make them create a kind of \emph{shortcut}  during adversarial training. Hence the neural network may fit this adversarial noises rather than the true features. This may partially answer why adversarial training is not that efficient for learning original features, which usually leads to deteriorated performance on clean test data. We think that such a distributional perspective of adversarial noises calls for further study to  understand the difficulty of adversarial learning or to improve the current adversarial training algorithms.

\newpage

\appendix

\newpage
\section{Some Proofs in Theorem \ref{thm:initialization-results}} \label{app:thm:initialization-results}

\subsection{The Independent Case}
We first prove the high probability bound for the case: {$\mD_x$ is  random and independent from all others $\{\mW, \va, \vx\}$.}, which is restated as the following lemma.
\begin{lemma}
Suppose that $\mW\in \sR^{m\times d}$ whose entries are i.i.d. sampled from $\gN(0,1/d)$, $\va\in \sR^{m}$ whose entries are i.i.d. sampled from $\gN(0,1/m)$, $\mD$ is a diagonal matrix whose diagonal entries are i.i.d., sampled from $Bernoulli(\frac{1}{2})$. We further assume that $\va, \mW$ and  $\mD$ are mutually independent. Then we have with probability at least $1-3Cn(e^{-c_1 d}+e^{-c_2 m})$ where $C, c_1, c_2$ are some constants, 
\begin{flalign}
\va^\top\mW\mW^\top \mD \va > \frac{1}{32}. 
\end{flalign}
\end{lemma}

\begin{proof}
We note that 
\begin{flalign}
 \va^\top\mW \mW^\top \mD \va = \va^\top\mD\mW \mW^\top \mD \va + \va^\top(\mI-\mD)\mW \mW^\top \mD \va.
\end{flalign}

For the first term, $\va^\top\mD\mW \mW^\top \mD \va=\norm{\mW^\top \mD \va}^2$.  Given $\mD\va$, we have $\mW^\top \mD \va \sim \gN\left(0, \frac{\norm{\mD\va}^2}{d}\mI_{d\times d}\right)$ and hence $ \norm{\mW^\top \mD \va}^2 \stackrel{d}{=} \frac{\norm{\mD\va}^2}{d}\chi_d^2$. 

We need a bound on the tail probability of $\chi_d^2$. \begin{lemma}
Suppose $X\sim \chi_d^2 $, i.e., chi square distribution with freedom $d$. Then  we have
\begin{flalign}
&\sP\{X < zd\}\leq (ze^{1-z})^{d/2}, \;\; \text{for } z<1,\\
 &\sP\{X > zd\}\leq (ze^{1-z})^{d/2}, \;\; \text{ for } z>1.
\end{flalign}
\end{lemma}

Hence 
\begin{flalign}
 \sP\left\{\norm{\mW^\top \mD \va}^2> \frac{\norm{\mD\va}^2}{d} \cdot \frac{d}{2}\right\}\ge 1- \left(\frac{e^{1/2}}{2}\right)^{d/2}>1- e^{-m/11}.
\end{flalign}


We note that $\norm{\va}^2 \sim \frac{1}{m} \chi_m^2$. Hence $\sP\{\norm{\va}^2< z\}\le(ze^{1-z})^{m/2}$ for $z<1$ and $\sP\{\norm{\va}^2>z\}\le(ze^{1-z})^{m/2}$ for $z>1$.

The diagonal entries of $\mD$ are Bernoulli random variables, and hence $\Tr(\mD)$ is a Binomial random variable with parameter $(m,\frac{1}{2})$. Due to the Hoeffding-type tail bound of Binomial random variable, we have for $z<1/2$
\begin{flalign}
\sP\{\Tr(\mD) < zm\} < \exp\left(-2m\left(\frac{1}{2}-z\right)^2\right).
\end{flalign}

Define an event $E_1:=\{\Tr(\mD)>\frac{1}{4}m\}$ and then its probability is at least $1-e^{-m/8}$. On event $E_1$, we can show $\sP\{\norm{\mD\va}^2>1/8\}>1-e^{-(\log\sqrt{2}-\frac{1}{4})m}> 1-e^{-m/11}$. Then define another event $E_2 := \{\norm{\mD\va}^2>1/8\}$ whose probability is at least $1-e^{-m/8}-e^{-m/11}$. 

Hence for the first term we have with probability at least $1-e^{-m/8}-2e^{-m/11}$
\begin{flalign}
\va^\top \mD\mW\mW^\top \mD\va >\frac{1}{16}.
\end{flalign}

For the second term,  let $D$ denote the set of index $j$ that $\mD_{j,j} =1$,  $\bar{D}$ denote the set of index $j$ that $\mD_{j,j} =0$ and $\vw_k$ denote the vector of the $k$-th column of $\mW$. Given $\{\mD, \va\}$ we have 
\begin{flalign}
\va^\top(\mI-\mD)\mW \mW^\top \mD \va = \sum_{k=1}^d (\vw_{k, \bar{D}}^\top\va_{\bar{D}})(\vw_{k,D}^\top\va_D) 
\end{flalign}
We note that $\vw_{k, \bar{D}}^\top\va_{\bar{D}} \sim \gN(0, \frac{\norm{\va_{\bar{D}}}^2}{d})$ and $\vw_{k,D}^\top\va_D\sim \gN(0, \frac{\norm{\va_{D}}^2}{d})$. They are independent from each other and their product is a sub-exponential random variable, with sub-exponential norm $K= \frac{2\norm{\va_{\bar{D}}}\norm{\va_{D}}}{\pi d}$.

\begin{definition}
The sub-exponential norm of $X$ is defined to be 
\begin{flalign}
\norm{X}_{\psi_1} = \sup_{p\ge 1}p^{-1}(\E|X|^p)^{1/p}.
\end{flalign}
\end{definition}
For the sum of sub-exponential random variables, we have the following Bernstein-type  bound.
\begin{lemma}[Corollary 5.17 in \citep{vershynin2010introduction}]
Let $X_1, ..., X_N$ be independent centered sub-exponential random variables, and let $K = \max_i \norm{X_i}_{\psi_1}$. Then, for every $\epsilon\ge 0$, we have
\begin{flalign}
\sP\left\{\left|\sum_{i=1}^N X_i\right|\ge \epsilon N\right\} \le 2 \exp\left[-c\min\left(\frac{\epsilon^2}{K^2}, \frac{\epsilon}{K}\right)N\right]
\end{flalign}
where $c>0$ is an absolute constant.
\end{lemma}

Using the sub-exponential Bernstein inequality, we have 
\begin{flalign}
\sP\left\{\left|\sum_{k=1}^d (\vw_{k, \bar{D}}^\top\va_{\bar{D}})(\vw_{k,D}^\top\va_D) \right|\ge \epsilon d\right\} \le 2 \exp\left[-c\min\left(\frac{\epsilon^2}{K^2}, \frac{\epsilon}{K}\right)d\right]. \label{eq:subexponential}
\end{flalign}

Define an event $E_3 := \{\norm{\va}^2< 2\}$ whose probability is at least $1-e^{-(0.5-\log \sqrt{2})m}>1-e^{-m/7}$.

On the intersection of $E_2$ and $E_3$, we have $\norm{\va_{D}}^2\ge \frac{\norm{\va}^2}{16}$ and hence $\frac{\norm{\va_{D}}^2}{\norm{\va_{\bar{D}}}^2}\ge \frac{1}{15}$, whose probability is at least $1-e^{-m/8}-e^{-m/11} -e^{-m/7}>1-3e^{-m/7}$.

On the event of $E_2 \cap E_3$ and taking $\epsilon = \frac{\norm{\mD \va}^2}{4d}$, the probability in \Eqref{eq:subexponential} is smaller than $2\exp\left[-c \frac{\pi^2 d}{960}\right]$.

Hence for the second term, we have with probability at least $1-3e^{-m/7} -2e^{-c\pi^2d/960}$,
\begin{flalign}
 |\va^\top(\mI-\mD)\mW \mW^\top \mD \va|\le \frac{1}{32}.
\end{flalign}

Hence combining with the bound on the first term, we have that $\va^\top\mW\mW^\top \mD \va \ge \frac{1}{32}$ holds with probability at least $1-e^{-m/8}-2e^{-m/11}-3e^{-m/7} -2e^{-c\pi^2d/960}$.  By taking the union bound over the training sample $n$, we complete the proof.

\end{proof}

\subsection{Proof of Lemma \ref{lem:innerproduct}}
\begin{proof}
We note that $\vh \sim \gN(0, \mI_{m})$. Hence because of the tail bound of the $\chi_m^2$, we have 
\begin{flalign}
\sP\{\norm{\vh}^2\le 2m \}> 1- (2e^{-1})^{m/2}> 1-e^{-m/7}.
\end{flalign}
Given $\vh$, we have $\va^\top\vh \sim \gN(0, \norm{\vh}^2/m)$. On the event of $\{\norm{\vh}^2\le 2m\}$, for some constant $c_2$, $\sP\{|\va^\top \vh|<\sqrt{c_2 d}\}> 1-2 \Phi(-\sqrt{c_2 d/2})>1-2e^{-c_2 d/4}$. Hence we have 
\begin{flalign}
|\va^\top\vh| < \sqrt{c_2d}
\end{flalign}
holds with probability at least $1-e^{-m/7}- 2e^{-c_2 d/4}$. 

On the event of $\{\norm{\vh}^2\le 2m\}$,  we have that $\sP\{\norm{\mD_x\vh}^2\le 2m \} =1$ and $\sP\{|\va^\top \mD_x\vh|<\sqrt{c_2 d}\}>1-2e^{-c_2 d/4}$. 

Combining the above two terms together, we complete the proof. 
\end{proof}

\end{document}